\documentclass[letterpaper]{article} 
\usepackage{aaai25}  
\usepackage{times}  
\usepackage{helvet}  
\usepackage{courier}  
\usepackage[hyphens]{url}  
\usepackage{graphicx} 
\usepackage{natbib}  
\usepackage{caption} 
\usepackage{algorithm}
\usepackage{algorithmic}

\usepackage{amsmath}
\usepackage{amssymb}
\usepackage{booktabs}
\usepackage{makecell}
\usepackage{tabularx}
\usepackage{amsthm}

\usepackage{newfloat}
\usepackage{listings}
\DeclareCaptionStyle{ruled}{labelfont=normalfont,labelsep=colon,strut=off} 
\floatstyle{ruled}
\newfloat{listing}{tb}{lst}{}
\floatname{listing}{Listing}

\newtheorem{theorem}{Theorem}

\newcommand{\cog}{\textsc{COG}}           
\newcommand{\model}{\textsc{CogGNN}}      
\newcommand{\stc}{\textsc{STC}}           
\newcommand{\cim}{\textsc{CIM}}           
\newcommand{\precisionk}{\mathrm{Precision@K}}

\title{EvoMail: Self-Evolving Cognitive Agents for Adaptive Spam and Phishing Email Defense}

\author{
Wei Huang\textsuperscript{\rm 1}\thanks{Equal contribution.}\,
\quad De-Tian Chu\textsuperscript{\rm 2}\footnotemark[1]\,
\quad Lin-Yuan Bai\textsuperscript{\rm 2}\,
\quad Wei Kang\textsuperscript{\rm 2}\,
\quad Hai-Tao Zhang\textsuperscript{\rm 2}\,
\quad Bo Li\textsuperscript{\rm 2}\,
\quad Zhi-Mo Han\textsuperscript{\rm 3}\,
\quad Jing Ge\textsuperscript{\rm 2}\,
\quad Hai-Feng Lin\textsuperscript{\rm 2}\thanks{Corresponding author: hfling@compintell.cn}
\\
\textsuperscript{\rm 1}\;People's Liberation Army 77606 Unit, Lasa 850000, China\\
\textsuperscript{\rm 2}\;Field Engineering College, Army Engineering University of PLA, Nanjing 210007, China\\
\textsuperscript{\rm 3}\;College Of Software, ZhengZhou University of Light Industry, Zhengzhou 450000, China\\
\texttt{hwazyx2013@163.com, Detian\_chu@aeu.edu.cn, linyuan\_bai@aeu.edu.cn, w\_kang2023@163.com, kirby8zhang@163.com,  13999932025@163.com, 542313460107@zzuli.edu.cn, 19822649042@163.com}
}

\begin{document}
\maketitle

\begin{abstract}
Modern email spam and phishing attacks have evolved far beyond keyword blacklists or simple heuristics, now focusing on complex intent exploitation that traditional systems struggle to understand \citep{Xin2025Enhancing}. Adversaries now craft multi-modal campaigns that combine natural-language text with obfuscated URLs, forged headers, and malicious attachments, adapting their strategies within days to bypass filters\citep{Lin2021Phishpedia}\citep{Wang2025EnhancingVLM}. Traditional spam detection systems, which rely on static rules or single-modality models, struggle to integrate heterogeneous signals or to continuously adapt, leading to rapid performance degradation.

We propose EvoMail, a self-evolving cognitive agent framework for robust detection of spam and phishing, building on recent advances in agent-based systems with memory and reflection capabilities \citep{Zhou2024CalibratedVLM}\citep{LIANG2025130470}. EvoMail first constructs a unified heterogeneous email graph that fuses textual content, metadata (headers, senders, domains), and embedded resources (URLs, attachments). A Cognitive Graph Neural Network (\model) enhanced by a Large Language Model (LLM) performs context-aware reasoning across these sources to identify coordinated spam campaigns. Most critically, EvoMail engages in an adversarial self-evolution loop: a ``red-team'' agent generates novel evasion tactics---such as character obfuscation or AI-generated phishing text\citep{Li2020TEXTSHIELD}---while the ``blue-team'' detector learns from failures, compresses experiences into a memory module\citep{Zhou2024CalibratedVLM}, and reuses them for future reasoning.

Extensive experiments on real-world datasets (Enron-Spam, Ling-Spam, SpamAssassin, and TREC) and synthetic adversarial variants demonstrate that EvoMail consistently outperforms state-of-the-art baselines in detection accuracy, adaptability to evolving spam tactics, and interpretability of reasoning traces. These results highlight EvoMail's potential as a resilient and explainable defense framework against next-generation spam and phishing threats.
\end{abstract}


\section{Introduction}
Email remains a universal communication backbone but also the primary vector for spam and phishing\citep{Goenka2024PhishingSurvey}. Contemporary campaigns are \emph{multi-source and rapidly adaptive\citep{Rossi2020TGN}}: natural and fluent text, obfuscated or shortened URLs, forged headers (SPF/DKIM/DMARC evasions), and polymorphic attachments. As filters are updated, attackers leverage powerful generative models \citep{liu2025luminamgpt} to mutate content within days via prompt-engineering or template-based generation. Traditional rule-based or single-modality classifiers\citep{Xiong2024SearchLLM} (e.g., keyword lists, Naive Bayes) perform well on known patterns but generalize poorly under tactic drift.

Advanced neural methods have improved robustness, yet three gaps persist: (i) \textbf{Heterogeneous fusion}: indicators span text, headers, domains, and attachments; modeling subtle cross-modal correlations\citep{Wang2024RNAErnie}\citep{Li2023LtrGCN} remains challenging\citep{Wang2025EnhancingVLM}\citep{Xin2024MmAP}. (ii) \textbf{Adaptation}: distribution shifts\citep{Li2023COLTR} from obfuscation, domain rotation, and AI-generated content quickly erode static models, a challenge addressed by recent domain adaptation methods\citep{Wang2024ContinualSurvey}\citep{Kirkpatrick2017EWC}\citep{Du2024}. (iii) \textbf{Contextual memory}: detectors rarely retain and reuse reasoning traces\citep{Li2023COLTR}\citep{Lyu2025Rethinking}, limiting recognition of variants of previously observed attacks.

We present \textbf{\cog}, a self-evolving cognitive agent that mimics human analysts\citep{Ren2024LLM-Graph-Survey}\citep{Li2023MRLtr}. \cog~unifies multi-source signals into a heterogeneous email graph and applies an LLM-enhanced \textbf{\model} for semantic-structural reasoning. A \emph{red--blue} adversarial loop continuously surfaces novel evasion tactics\citep{Li2023LtrGCN} and converts detection failures into compressed, reusable experiences, injected back into reasoning and fine-tuning. Our contributions:
\begin{itemize}
    \item \textbf{Self-evolving agent for spam/phishing:} a closed-loop framework combining heterogeneous email graphs\citep{Li2024GS2P}, LLM-augmented GNN reasoning, adversarial generation, and experience compression\citep{Li2024GS2P}.
    \item \textbf{\model:} an LLM-guided semantic attention and query optimizer that couples textual semantics with graph structure across content, headers, domains, and attachments.
    \item \textbf{Experience compression \& reuse:} distills failed traces into structured heuristics to stabilize adaptation under distributional and adversarial shifts.
    \item \textbf{Comprehensive evaluation:} static, shift, and cross-modal settings on public corpora and synthetic adversarial variants; we report accuracy/F1 as well as \stc~(structured temporal consistency) and \cim~(cognitive interpretability) for regulator-aligned transparency.
\end{itemize}
\begin{figure*}[t]
\centering
\includegraphics[width=0.95\linewidth]{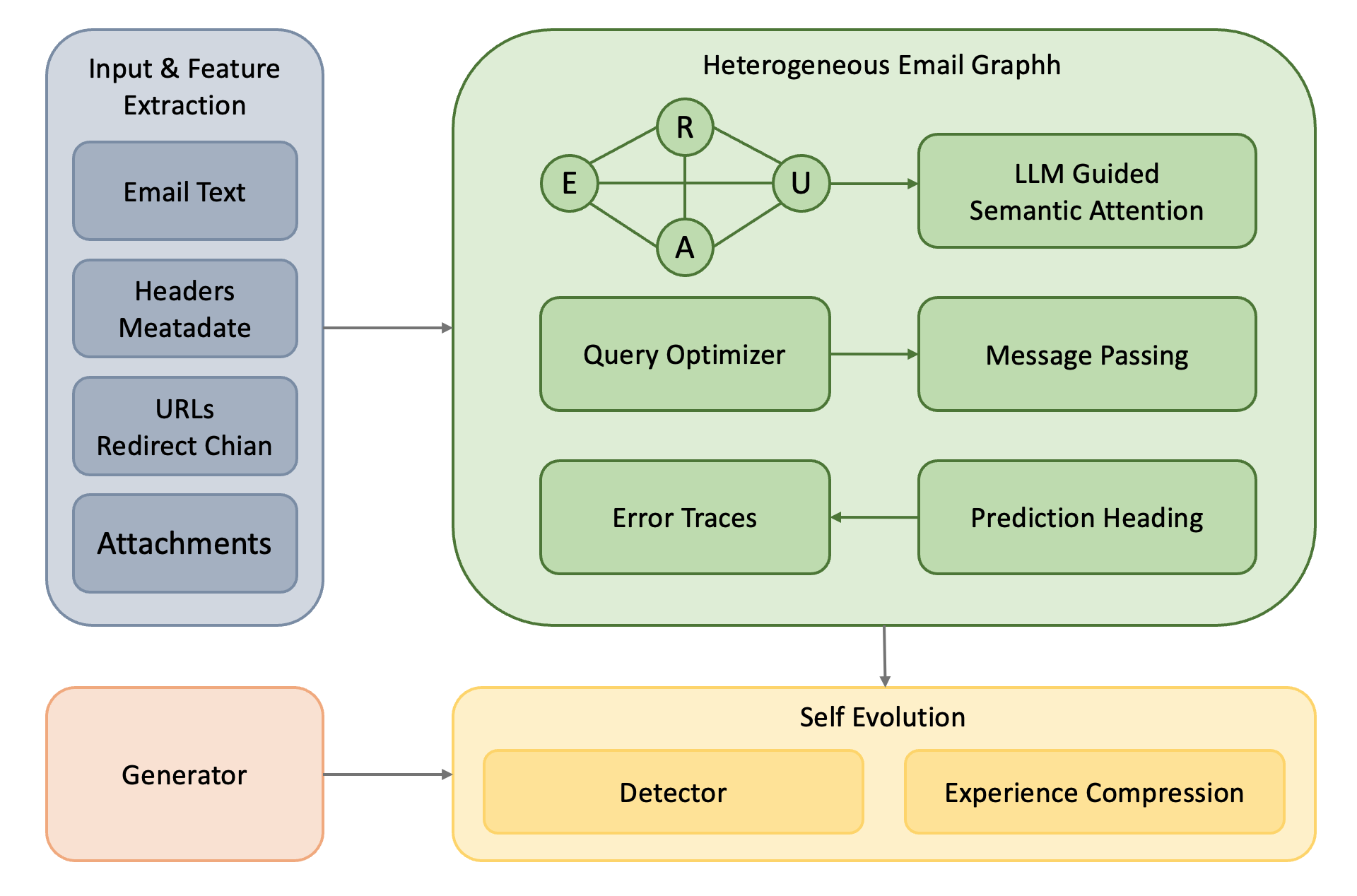} 
\caption{Schematic: heterogeneous email graph and high-attention evidence path
.}
\label{fig:schematic}
\end{figure*}
\section{Related Work}
\paragraph{Email spam and phishing detection.}
Classical filters (Bayes\citep{Metsis2006Spam}, SVM\citep{Fette2007Learning}, rule lists) and modern neural models (TextCNN\citep{Kim2014CNN}, Transformers) predominantly focus on content, with metadata and URLs often auxiliary. Graph-based perspectives connect senders, recipients, domains, and URLs but struggle to fuse semantics and structure end-to-end. \citep{Wu2021Survey} \citep{Li2025Anomaly}\citep{Schlichtkrull2017R-GCN}

\paragraph{LLMs for security text.}
LLMs excel at contextual reasoning and pattern abstraction, and have been successfully applied to tasks like phishing email detection\citep{Devlin2019BERT}\citep{Brown2020GPT3}\citep{Zhao2025SurveyLLM}; yet without persistent memory and guardrails, they can be brittle to prompt injection or novel obfuscation\citep{Wei2023Jailbroken}\citep{Zhou2024Hallucination}, and their inference efficiency remains a consideration \citep{wan2025d2o}. Integrating LLMs with structured models (graphs\citep{Zhao2023GLEM}\citep{Ede2022DEEPCASE}) is a promising path\citep{Wan2025SRPO} toward richer evidence aggregation\citep{Liang2021GNNSurvey}\citep{Qian2023LLM-Molecule}.

\paragraph{Self-evolving systems.}
Self-play\citep{Silver2016AlphaGo}, reflective memory, and inference-time optimization through reflection \citep{Zhuo2025Reflection} have shown promise in dynamic domains, alongside online prompt/model updates\citep{Xin2024VMT-Adapter} for tasks ranging from game playing to automated cyber defense\citep{Palmer2024CyberSurvey}. For email security, realizing safe, auditable, and data-efficient self-evolution remains under-explored\citep{Zhou2024Hallucination}\citep{Thapa2023FL-Phishing}.

\section{Problem Formulation}
Let $\mathcal{D}=\{D^{\text{text}}, D^{\text{meta}}, D^{\text{url}}, D^{\text{att}}\}$ denote email content, metadata (headers/senders/receivers/domains), URLs, and attachments. We build a heterogeneous graph $G=(V,E,X)$ with node types (emails, senders, receivers, domains, URLs, attachments) and relation types (\texttt{sent\_to}, \texttt{hosted\_on}, \texttt{contains}, \texttt{linked\_to}, \texttt{replied\_to}) \citep{Wang2021HGAN}. Each node $v\in V$ has features $x_v$ (e.g., text embeddings, header signals, URL features, file hashes/MIME).

Given an email (or email thread) $q$, the task is to predict $y\in\{0,1\}$ (ham vs.\ spam/phishing):
\begin{equation}
    \hat{y} = f_\theta\!\big( q, G \big) \in [0,1],
\end{equation}
where $f_\theta$ integrates local content and global graph context.

\section{Method: COG Framework}

\subsection{Overview}

Figure~\ref{fig:schematic} provides an overview of the \textbf{COG} framework. 
Incoming emails are represented as nodes of different types (emails, senders, recipients, domains, URLs, attachments), with edges capturing semantic and structural relations such as \texttt{sent-to}, \texttt{contains}, \texttt{hosted-on}, and \texttt{linked-to}. 
These nodes and edges together form a heterogeneous email graph, over which the LLM-enhanced CogGNN performs multi-source reasoning. 
Highlighted edges in the figure indicate a high-attention \emph{evidence path}, showing how EvoMail aggregates suspicious cues---e.g., forged headers, anomalous domains, or obfuscated URLs---into an interpretable reasoning trace that explains why an email is classified as spam or phishing. 
The full framework integrates four modules: (1) heterogeneous graph construction, (2) CogGNN reasoning with LLM-guided attention, (3) self-evolution via red--blue adversarial training and experience compression\citep{Khosla2023MemorySurvey}, and (4) explainable evidence-path backtracing.

\subsection{Notation and Definitions}

We summarize the notation used throughout this section:
\begin{itemize}
    \item $\mathcal{G} = (\mathcal{V}, \mathcal{E}, \mathcal{R})$: heterogeneous email graph with nodes $\mathcal{V}$, edges $\mathcal{E}$, and relation types $\mathcal{R}$.
    \item $v_i \in \mathcal{V}$: email node $i$ with raw feature vector $x_i \in \mathbb{R}^d$.
    \item $h_v^{(0)} \in \mathbb{R}^{d_h}$: initial embedding of node $v$ after feature transformation.
    \item $h_v^{(k)} \in \mathbb{R}^{d_h}$: hidden representation of node $v$ at GNN layer $k \in \{1,\ldots,L\}$.
    \item $\mathcal{N}(v)$: full neighborhood of node $v$; $\mathcal{N}_K(v) \subseteq \mathcal{N}(v)$: top-$K$ selected neighbors.
    \item $e_{uv} \in \mathcal{E}$: edge between nodes $u$ and $v$ with relation type $r_{uv} \in \mathcal{R}$.
    \item $\alpha_{uv}^{(k)} \in [0,1]$: attention weight from node $u$ to $v$ at layer $k$.
    \item $f_\theta: \mathbb{R}^d \rightarrow [0,1]$: detection model parameterized by $\theta$.
    \item $\mathcal{M}_t$: compressed memory at iteration $t$, storing tuples $(e,\hat y, \text{trace})$.
    \item $\mathcal{A}_t = \{E_{\text{adv}}^{(m)}\}_{m=1}^{M_t}$: set of adversarial samples at iteration $t$.
    \item $\phi: \mathcal{V} \rightarrow \mathbb{R}^{d_{\phi}}$: feature extraction function mapping emails to embedding space.
    \item $P(u,v) \in \mathbb{R}^{d_p}$: contextualized prompt embedding for node pair $(u,v)$.
    \item Hyperparameters: $K$ (neighbors), $L$ (GNN layers), $\tau$ (temperature), $\beta$ (balance), $\lambda,\mu$ (loss weights), $M_{\max}$ (memory budget).
\end{itemize}

\subsection{Heterogeneous Graph Construction}

\paragraph{Multi-modal feature extraction.}
For each email $v$, we extract features from multiple modalities and construct the initial representation:
\begin{align}
    x_v^{\text{text}} &= \text{TF-IDF}(\text{subject}(v)) \oplus \text{TF-IDF}(\text{body}(v)) \in \mathbb{R}^{d_t}, \\
    x_v^{\text{meta}} &= \Bigl[
        \begin{aligned}[t]
            &\text{hour}(v),\; \text{weekday}(v), \\
            &\text{length}(v),\; \text{attach\_count}(v)
        \end{aligned}
    \Bigr] \in \mathbb{R}^{d_m}, \\
    x_v^{\text{network}} &= \Bigl[
        \begin{aligned}[t]
            &\text{sender\_rep}(v),\; \text{domain\_age}(v), \\
            &\text{url\_count}(v)
        \end{aligned}
    \Bigr] \in \mathbb{R}^{d_n}, \\
    x_v &= x_v^{\text{text}} \oplus x_v^{\text{meta}} \oplus x_v^{\text{network}} \in \mathbb{R}^d
\end{align}
where $\oplus$ denotes vector concatenation, $d = d_t + d_m + d_n$, and each component captures different aspects of email characteristics.

\paragraph{Relation modeling.}
We define multiple relation types in $\mathcal{R} = \{r_1, r_2, \ldots, r_{|\mathcal{R}|}\}$:
\begin{align}
    r_{\text{domain}}(u,v) &= \mathbb{I}[\text{domain}(u) = \text{domain}(v)] \cdot w_{\text{domain}}, \\
    r_{\text{temporal}}(u,v) &= 
    \begin{aligned}[t]
        &\exp\!\Biggl(
        -\frac{|\text{timestamp}(u) - \text{timestamp}(v)|}{\sigma_t}
        \Biggr) \\
        &\cdot \; w_{\text{temporal}}
    \end{aligned}, \\
    r_{\text{semantic}}(u,v) &= \cos(\text{BERT}(u), \text{BERT}(v)) \cdot w_{\text{semantic}}, \\
    r_{\text{sender}}(u,v) &= \mathbb{I}[\text{sender}(u) = \text{sender}(v)] \cdot w_{\text{sender}}
\end{align}
where $\mathbb{I}[\cdot]$ is the indicator function, $\sigma_t > 0$ controls temporal decay, and $w_{\text{type}} > 0$ are learned relation weights.

\paragraph{Edge construction.}
An edge $e_{uv}$ exists if any relation exceeds threshold $\epsilon_r$:
\begin{equation}
\begin{split}
    e_{uv} \in \mathcal{E} \iff &\; \max_{r \in \mathcal{R}} r(u,v) > \epsilon_r, \\
    &\text{with weight } 
    W_{uv} = \sum_{r \in \mathcal{R}} r(u,v).
\end{split}
\end{equation}

\subsection{CogGNN: LLM-Enhanced Cognitive Graph Neural Network}

\paragraph{Initial embedding transformation.}
Raw features are transformed into hidden space via learnable projection:
\begin{equation}
    h_v^{(0)} = \text{LayerNorm}(\sigma(W_{\text{init}} x_v + b_{\text{init}})) \in \mathbb{R}^{d_h}
\end{equation}
where $W_{\text{init}} \in \mathbb{R}^{d_h \times d}$, $b_{\text{init}} \in \mathbb{R}^{d_h}$, and $\sigma(\cdot)$ is the ReLU activation.

\paragraph{Contextual prompt generation.}
For each node pair $(u,v)$, we construct a structured prompt and encode it:
\begin{align}
    \text{prompt\_str}(u,v) &= 
    \textsc{Template}\Bigl(
        \begin{aligned}[t]
            &\text{desc}(u), \; \text{desc}(v), \\
            &\text{relation}(e_{uv}), \; \text{task\_context}
        \end{aligned}
    \Bigr), \\
    P(u,v) &= \text{LLM}_{\text{encoder}}(\text{prompt\_str}(u,v)) \in \mathbb{R}^{d_p}
\end{align}
where $\text{LLM}_{\text{encoder}}$ produces contextualized embeddings capturing semantic relationships.

\paragraph{Adaptive neighbor selection.}
We compute a composite salience score integrating multiple signals:
\begin{align}
    s_{\text{struct}}(u,v) &= 
    \begin{aligned}[t]
        &\text{PageRank}(u) 
        + \frac{\text{deg}(u)}{\max_i \text{deg}(v_i)} \\
        &+ \frac{1}{1 + \text{shortest\_path}(u,v)}
    \end{aligned}, \\
    s_{\text{freq}}(u,v) &= \frac{\text{co\_occurrence}(u,v)}{\sqrt{\text{total\_count}(u) \cdot \text{total\_count}(v)}}, \\
    s_{\text{semantic}}(u,v) &= \frac{h_u^{(0)} \cdot h_v^{(0)}}{\|h_u^{(0)}\|_2 \|h_v^{(0)}\|_2}, \\
    S(u,v) &= w_s \cdot s_{\text{struct}}(u,v) + w_f \cdot s_{\text{freq}}(u,v) + \\ &w_c \cdot s_{\text{semantic}}(u,v)
\end{align}
where $[w_s, w_f, w_c] = \text{softmax}([a_s, a_f, a_c])$ with learnable parameters $a_s, a_f, a_c \in \mathbb{R}$. The top-$K$ neighbors are selected as:
\begin{equation}
\begin{aligned}
    \mathcal{N}_K(v) &= \{u_1, u_2, \ldots, u_K\} \\
    &\text{where } S(u_1,v) \geq S(u_2,v) \geq \ldots \geq S(u_K,v).
\end{aligned}
\end{equation}

\paragraph{LLM-guided attention mechanism.}
Building upon the principles of graph attention networks\citep{Velickovic2018GAT}, for each layer $k$, attention weights integrate semantic and structural reasoning\citep{Yang2025WCDT}:
\begin{align}
    \text{llm\_score}(u,v) &= \text{MLP}_{\text{attn}}(P(u,v)) \in [0,1], \\
    \text{struct\_features}(u,v) &= 
    \begin{aligned}[t]
        [&w_r^{\top} \text{onehot}(r_{uv}), \; \log(1+\text{deg}(u)), \\
         &\log(1+\text{deg}(v)), \; \tfrac{1}{1+\text{spath}(u,v)}]
    \end{aligned}, \\
    \text{struct\_score}(u,v) &= w_{\text{struct}}^{\top} \text{struct\_features}(u,v), \\
    e_{uv}^{(k)} &= 
    \begin{aligned}[t]
        &\text{llm\_score}(u,v) + \beta \cdot \text{struct\_score}(u,v) \\
        &+ \gamma \cdot 
        \frac{h_u^{(k-1)} \cdot h_v^{(k-1)}}{\|h_u^{(k-1)}\|_2 \, \|h_v^{(k-1)}\|_2}
    \end{aligned}, \\
    \alpha_{uv}^{(k)} &= \frac{\exp(e_{uv}^{(k)}/\tau)}{\sum_{j \in \mathcal{N}_K(v)} \exp(e_{jv}^{(k)}/\tau)}
\end{align}
where $\text{MLP}_{\text{attn}}$ is a multi-layer perceptron, $w_r \in \mathbb{R}^{|\mathcal{R}|}$ are relation embedding weights, $w_{\text{struct}} \in \mathbb{R}^4$, and $\gamma \geq 0$ controls the influence of current representations.

\paragraph{Multi-layer message passing.}
Node representations evolve through $L$ layers of message aggregation:
\begin{align}
    m_v^{(k)} &= \sum_{u \in \mathcal{N}_K(v)} \alpha_{uv}^{(k)} \cdot (W_{\text{neigh}}^{(k)} h_u^{(k-1)} + W_{\text{edge}}^{(k)} \text{embed}(r_{uv})), \\
    \tilde{h}_v^{(k)} &= W_{\text{self}}^{(k)} h_v^{(k-1)} + m_v^{(k)} + b^{(k)}, \\
    h_v^{(k)} &= \text{LayerNorm}(\sigma(\tilde{h}_v^{(k)})) + \text{Dropout}(h_v^{(k-1)})
\end{align}
where $W_{\text{neigh}}^{(k)}, W_{\text{self}}^{(k)}, W_{\text{edge}}^{(k)} \in \mathbb{R}^{d_h \times d_h}$ are layer-specific transformation matrices, $\text{embed}(r_{uv}) \in \mathbb{R}^{d_h}$ embeds relation types, and we include residual connections for training stability\citep{He2015ResNet}.

\paragraph{Final prediction.}
The detection score combines multi-layer representations:
\begin{align}
    z_v &= \text{Concat}([h_v^{(1)}, h_v^{(2)}, \ldots, h_v^{(L)}]) \in \mathbb{R}^{L \cdot d_h}, \\
    \hat{y}_v &= \text{sigmoid}(W_{\text{out}} z_v + b_{\text{out}})
\end{align}
where $W_{\text{out}} \in \mathbb{R}^{1 \times L \cdot d_h}$ and $b_{\text{out}} \in \mathbb{R}$.

\subsection{Self-Evolution via Red--Blue Adversarial Training}

\paragraph{Red team: Adversarial sample generation.}
The red team generates adversarial emails through gradient-based and semantic perturbations\citep{Goodfellow2015Adversarial}\citep{Ganguli2022RedTeaming}:
\begin{align}
    E_{\text{grad}} &= E_{\text{seed}} + \epsilon \cdot \text{sign}(\nabla_{E_{\text{seed}}} \log f_\theta(E_{\text{seed}})), \\
    E_{\text{semantic}} &= \text{SemanticMutate}(E_{\text{seed}}, \mathcal{V}_{\text{vocab}}, \rho_{\text{mut}}), \\
    E_{\text{hybrid}} &= \text{Combine}(E_{\text{grad}}, E_{\text{semantic}}, \lambda_{\text{hybrid}})
\end{align}
where $\epsilon > 0$ controls perturbation magnitude, $\mathcal{V}_{\text{vocab}}$ is the vocabulary for semantic mutations, $\rho_{\text{mut}} \in [0,1]$ is mutation rate, and $\lambda_{\text{hybrid}} \in [0,1]$ balances the combination.

The red team optimizes a multi-objective reward:
\begin{align}
    \text{Novelty}(E_{\text{adv}}) &= \min_{(e_i,\hat{y}_i,t_i) \in \mathcal{M}_t} \|\phi(E_{\text{adv}}) - \phi(e_i)\|_2, \\
    \text{Evasion}(E_{\text{adv}}) &= \max(0, 0.5 - f_\theta(E_{\text{adv}})), \\
    \text{Complexity}(E_{\text{adv}}) &= \frac{\|\phi(E_{\text{adv}}) - \phi(E_{\text{seed}})\|_2}{\|\phi(E_{\text{seed}})\|_2}, \\
    R_{\text{red}}(E_{\text{adv}}) &= 
    \begin{aligned}[t]
        &w_n \cdot \text{Novelty}(E_{\text{adv}}) 
        + w_e \cdot \text{Evasion}(E_{\text{adv}}) \\
        &- w_c \cdot \text{Complexity}(E_{\text{adv}})
    \end{aligned}
\end{align}
where $w_n + w_e + w_c = 1$ and $w_n, w_e, w_c \geq 0$.

\paragraph{Blue team: Failure analysis.}
The blue team detects failures and extracts detailed traces:
\begin{align}
    \mathcal{F}_t &= \{E \in \mathcal{A}_t : f_\theta(E) < \delta_{\text{fail}} \text{ and } \text{ground\_truth}(E) = 1\}, \\
    \text{Trace}(E_{\text{fail}}) &= \{h_v^{(k)}, \{\alpha_{uv}^{(k)}\}_{u \in \mathcal{N}_K(v)}, \text{path}(v)\}_{k=1}^L
\end{align}
where $\delta_{\text{fail}} = 0.5$ is the failure threshold and $\text{path}(v)$ records the attention-based reasoning path.

\paragraph{Experience compression and memory management\citep{Lopez-Paz2022GEM}.}
Failed traces are clustered and compressed using $k$-medoids:
\begin{align}
    \text{dist}(E_1, E_2) &= 
    \begin{aligned}[t]
        &\|\phi(E_1) - \phi(E_2)\|_2 \\
        &+ \alpha_{\text{trace}} \cdot \text{trace\_dist}(\text{Trace}(E_1), \text{Trace}(E_2))
    \end{aligned}, \\
    \mathcal{C} &= 
    \begin{aligned}[t]
        &\text{KMedoids}(\mathcal{F}_t, k_t), \\
        &k_t = \min(|\mathcal{F}_t|,\, M_{\max} - |\mathcal{M}_t|)
    \end{aligned}, \\
    \tilde{e}_j &= 
    \begin{aligned}[t]
        &\arg\min_{e \in \mathcal{C}_j} 
        \sum_{e' \in \mathcal{C}_j} \text{dist}(e, e')
    \end{aligned}, \\
    \mathcal{M}_{t+1} &= 
    \begin{aligned}[t]
        &\text{LRU}\Bigl(\mathcal{M}_t \cup 
        \{(\tilde{e}_j, f_{\theta_t}(\tilde{e}_j), \\
        &\qquad \text{Trace}(\tilde{e}_j))\}_{j=1}^{|\mathcal{C}|},\, M_{\max}\Bigr)
    \end{aligned}
\end{align}
where $\alpha_{\text{trace}} \geq 0$ weights trace similarity, and LRU removes oldest entries when memory exceeds capacity.

\paragraph{Adversarial training objective.}
The total loss integrates task performance, memory consistency, and adversarial robustness:
\begin{align}
    \mathcal{L}_{\text{task}}(\theta) &= -\sum_{(v,y) \in \mathcal{D}} [y \log f_\theta(v) + (1-y) \log(1-f_\theta(v))], \\
    \mathcal{L}_{\text{cons}}(\theta) &= 
    -\sum_{(e,\hat{y},t) \in \mathcal{M}_t} 
    \begin{aligned}[t]
        [&\hat{y} \log f_\theta(e) \\
        &+ (1-\hat{y}) \log(1-f_\theta(e))]
    \end{aligned}, \\
    \mathcal{L}_{\text{adv}}(\theta) &= -\sum_{E \in \mathcal{A}_t} \log f_\theta(E) + \sum_{E \in \mathcal{A}_t} \log(1-f_\theta(E_{\text{benign}})), \\
    \mathcal{L}_{\text{reg}}(\theta) &= \|\theta\|_2^2, \\
    \mathcal{L}_{\text{total}}(\theta) &= \mathcal{L}_{\text{task}}(\theta) + \lambda \mathcal{L}_{\text{cons}}(\theta) + \mu \mathcal{L}_{\text{adv}}(\theta) + \nu \mathcal{L}_{\text{reg}}(\theta)
\end{align}
where $E_{\text{benign}}$ are benign emails, and $\nu > 0$ controls regularization strength.

\subsection{Explainable Reasoning}
To ensure transparency and auditability, a critical requirement for security applications, we develop an explainable reasoning module\citep{Ribeiro2016LIME}\citep{Lundberg2017SHAP}. Our explanation approach, which extracts the most influential evidence path, is one of several paradigms in graph explainability\citep{Vu2020PGM-Explainer}. It contrasts with counterfactual methods that identify minimal changes to flip a prediction, and with probabilistic approaches that generate a distribution over explanatory subgraphs. Furthermore, it differs from decomposition-based methods like Layer-wise Relevance Propagation (LRP)\citep{Bach2015LRP}, which explain predictions by propagating relevance scores backwards through the network.

\paragraph{Evidence path extraction with confidence scoring.}
Similar to methods developed for explaining GNNs\citep{Ying2019GNNExplainer}, we extract reasoning paths by following high-attention edges with confidence assessment:
\begin{align}
    \text{confidence}(v, k) &= \max_{u \in \mathcal{N}_K(v)} \alpha_{uv}^{(k)}, \\
    u_i^* &= \arg\max_{u \in \mathcal{N}_K(u_{i-1}^*)} \alpha_{u,u_{i-1}^*}^{(L-i+1)}, \\
    \text{Path}(v) &= 
    \begin{aligned}[t]
        &[ (v, \text{confidence}(v,L)), \\
        &\;(u_1^*, \text{confidence}(u_1^*,L-1)), \ldots, \\
        &\;(u_d^*, \text{confidence}(u_d^*,L-d)) ]
    \end{aligned}
\end{align}
where path extraction stops when $d \geq D_{\max}$ or $\text{confidence}(u_d^*, L-d) < \alpha_{\min}$.

\paragraph{Feature importance scoring.}
For each node in the reasoning path, we compute feature importance via gradient analysis:
\begin{align}
    \text{importance}(f_i, v) &= \left|\frac{\partial f_\theta(v)}{\partial x_{v,i}}\right| \cdot |x_{v,i}|, \\
    \text{top\_features}(v) &= \text{TopK}_i(\text{importance}(f_i, v), K_{\text{feat}})
\end{align}
where $x_{v,i}$ is the $i$-th feature of node $v$ and $K_{\text{feat}}$ controls the number of important features to report.

\paragraph{Natural language explanation generation\citep{Lucic2022CF-GNNExplainer}.}
Explanations combine path analysis with feature importance:
\begin{align}
    \text{path\_summary} &= 
    \begin{aligned}[t]
        &\textsc{Summarize}\big(\{(\text{type}(u_i), \\
        &\;\;\text{relation}(u_i, u_{i-1}), \text{confidence}(u_i))\}_{i=1}^d\big)
    \end{aligned}, \\
    \text{feature\_summary} &= \textsc{Describe}\!\left(\bigcup_{i=1}^d \text{top\_features}(u_i)\right), \\
    \text{Explanation}(v) &= 
    \begin{aligned}[t]
        &\textsc{Template}\big(\text{path\_summary}, \\
        &\;\;\text{feature\_summary}, f_\theta(v)\big)
    \end{aligned}
\end{align}

\subsection{Theoretical Analysis}

\begin{theorem}[Convergence Guarantee]
\label{thm:convergence}
Assume each loss component in $\mathcal{L}_{\text{total}}$ is $L$-Lipschitz continuous with respect to $\theta$. If the learning rate satisfies $\eta \leq 1/(2L)$, then gradient descent converges to an $\epsilon$-stationary point in $O(1/\epsilon^2)$ iterations.
\end{theorem}

\begin{proof}[Sketch]
The weighted sum of Lipschitz functions remains Lipschitz with constant $L_{\text{total}} = L + \lambda L + \mu L + \nu L = (1+\lambda+\mu+\nu)L$. Standard analysis of non-convex optimization \citep{ghadimi2013stochastic} shows that $\mathbb{E}[\|\nabla \mathcal{L}_{\text{total}}\|^2]$ decreases at rate $O(1/T)$.
\end{proof}

\begin{theorem}[Per-layer Time Complexity of \textsc{CogGNN}]
\label{thm:complexity}
Let $|\mathcal{V}|$ be the number of nodes, $L$ the number of layers, 
$K$ the Top-$K$ neighbors per node, $H$ the number of attention heads, 
and $d_h$ the head (or hidden) dimension. 
Let $C_{\text{LLM}}$ denote the cost of invoking the LLM on a node–neighbor pair.
Assuming sparse neighborhoods (max degree absorbed into $K$) and standard
query–key–value attention, the total time complexity over $L$ layers is
\[
\mathcal{O}\!\big(L \cdot |\mathcal{V}| \cdot (H K d_h \;+\; d_h^2 \;+\; K C_{\text{LLM}})\big),
\]
where $H K d_h$ accounts for attention score/aggregation, $d_h^2$ for linear
projections, and $K C_{\text{LLM}}$ for contextual LLM prompting per node.
\end{theorem}

\begin{theorem}[Attention Complexity]
\label{thm:complexity}
For a graph with $|\mathcal{V}|$ nodes, maximum degree $\Delta$, and $L$ layers, the time complexity per iteration is:
\begin{equation}
    O(L \cdot |\mathcal{V}| \cdot K \cdot d_h^2 + |\mathcal{V}| \cdot K \cdot C_{\text{LLM}})
\end{equation}
where $C_{\text{LLM}}$ is the cost of LLM inference per node pair.
\end{theorem}

\begin{proof}[Sketch]
For each layer and node $v$: 
(i) computing attention scores over $K$ neighbors with $H$ heads costs 
$\mathcal{O}(H K d_h)$ (dot-products and softmax); 
(ii) the standard linear projections (e.g., $Q,K,V$ and output) contribute 
$\mathcal{O}(d_h^2)$ per node; 
(iii) applying an LLM to each $(v,u)$ pair for contextual reasoning
adds $\mathcal{O}(K C_{\text{LLM}})$. 
Summing over $|\mathcal{V}|$ nodes and $L$ layers yields the bound. 
This matches common analyses of attention mechanisms \citep{vaswani2017attention} 
and message-passing GNNs \citep{zhou2020graph}.
\end{proof}

\begin{algorithm}[t]
\caption{COG Framework Training}
\label{alg:cog_training}
\begin{algorithmic}[1]
\REQUIRE Dataset $\mathcal{D}$, hyperparameters $\{K,L,\lambda,\mu,\nu,\eta,M_{\max}\}$
\STATE Initialize parameters $\theta$, memory $\mathcal{M}_0=\emptyset$
\STATE Construct heterogeneous graph $\mathcal{G}$ using Equations 
\FOR{$t=1$ to $T$}
    \STATE // Forward pass with current model
    \STATE Compute predictions $\{\hat{y}_v\}_{v \in \mathcal{V}}$ using Equations 
    
    \STATE // Red team adversarial generation  
    \STATE $\mathcal{A}_t \leftarrow \textsc{RedTeam.Generate}(f_\theta,\mathcal{M}_t)$ using Equations 
    
    \STATE // Blue team failure detection
    \STATE $\mathcal{F}_t \leftarrow \textsc{BlueTeam.Failures}(\mathcal{A}_t,f_\theta)$ using Equations 
    
    \STATE // Experience compression and memory update
    \STATE $\mathcal{M}_{t+1} \leftarrow \textsc{Compress}(\mathcal{M}_t \cup \mathcal{F}_t, M_{\max})$ using Equations 
    
    \STATE // Parameter update
    \STATE $\theta \leftarrow \theta - \eta \nabla_\theta \mathcal{L}_{\text{total}}(\theta)$ using Equations 
\ENDFOR
\RETURN trained model $f_\theta$, experience memory $\mathcal{M}_T$
\end{algorithmic}
\end{algorithm}

\section{Datasets}
As illustrated in Table~\ref{tab:data}, We evaluate on public corpora and synthetic adversarial variants:
\begin{itemize}
    \item \textbf{Enron-Spam}: real enterprise emails labeled ham/spam.
    \item \textbf{Ling-Spam}: linguistic spam dataset.
    \item \textbf{SpamAssassin}: curated spam/ham with headers.
    \item \textbf{TREC 2007/2008 Spam}: realistic benchmarking corpora\citep{Cormack2007TREC}.
    \item \textbf{Synthetic-Adversarial}: phases P1--P3 simulating tactic drift (keyword/template, Unicode/leet + domain rotation, AI-generated phishing with forged headers/attachments).
\end{itemize}

\begin{table}[t]
\centering
\caption{Dataset statistics.}
\label{tab:data}
\resizebox{\linewidth}{!}{%
\begin{tabular}{lrrrr}
\toprule
Dataset & \#Emails & Spam(\%) & Modalities & Notes \\
\midrule
Enron-Spam     & 33{,}000 & 49.3 & text+meta         & enterprise mix \\
Ling-Spam      & 2{,}893  & 16.6 & text              & ling.~bias ctrl \\
SpamAssassin   & 9{,}349  & 56.1 & text+meta         & rich headers \\
TREC Spam      & 75{,}000 & 50.0 & text+meta+url     & realistic dist. \\
Synthetic-Adv  & 60{,}000 & 50.0 & text+meta+url+att & P1--P3 evolving \\
\bottomrule
\end{tabular}}
\end{table}

\section{Experimental Setup}

\paragraph{Baselines.}
As elabrated in Table~\ref{tab:overall}, we compare against classical, neural, and graph-based methods:
\begin{itemize}
  \item \textbf{NB / SVM / RF}: TF-IDF + basic metadata\citep{Fette2007Learning}.
  \item \textbf{TextCNN / BERT}: content-only encoders\citep{Kim2014CNN}.
  \item \textbf{GraphSAGE + MLP}\citep{hamilton2018inductiverepresentationlearninglarge}: graph topology with shallow features (no LLM attention).
  \item \textbf{Early / Mid Fusion}: concatenation-based multi-modal baselines (input-level vs.\ intermediate-layer fusion).
  \item \textbf{EvoMail (ours)}: heterogeneous email graph + LLM-enhanced \model\ + self-evolving memory. Ablations:
  \emph{w/o Context} (remove LLM attention),
  \emph{w/o Query Optimizer} (static neighbor sampling),
  \emph{w/o Memory} (no failure-trace memory).
\end{itemize}

\paragraph{Evaluation Scenarios.}
\textbf{Static}: 80/20 train--test split. \;
\textbf{Shift}: train on P1; incremental updates on P2--P3 (no revisits), with 10\% novel P3 attacks. \;
\textbf{Cross-modal}: (1) text-only, (2) text+metadata, (3) full graph (text+metadata+URLs/attachments).

\paragraph{Metrics.}
Accuracy, Precision, Recall, F1, and $\precisionk$. 
We further report \textbf{\stc} (temporal alignment of attention-based evidence traces) and \textbf{\cim} (expert-rated explanation coherence/relevance/auditability; 0--1 scale). 
In dynamic settings, we track update latency and memory compression efficiency.

\paragraph{Implementation Details.}
Results are averaged over 5 runs (different seeds). The LLM attention head is tuned via LoRA\citep{Xin2024V-PETL}\citep{Liu2024SparseTuning}. The top-$K$ optimizer selects $K \in \{8,16,24,32\}$ on validation. Experiments use NVIDIA A100 (40GB). Additional hyperparameters are in the supplement\citep{Shchur2019Pitfalls}.

\section{Results and Analysis}
\begin{figure}[t]
    \centering
    \includegraphics[width=0.85\linewidth]{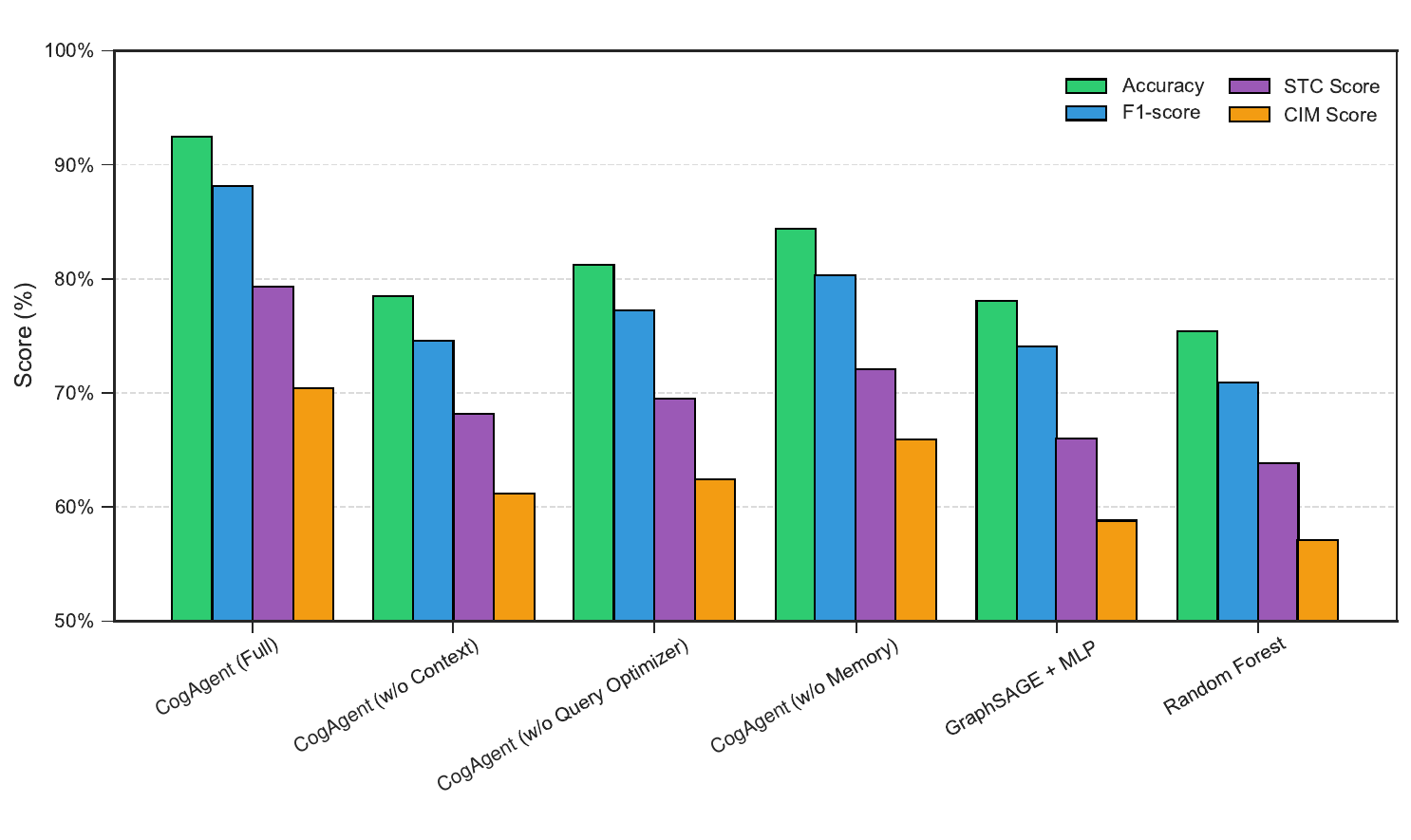}
    \caption{Overall performance comparison across models (NB, SVM, RF, TextCNN, BERT, GraphSAGE+MLP, Fusion baselines, EvoMail). 
    EvoMail consistently outperforms baselines in accuracy, F1, and interpretability (CIM).}
    \label{fig:overall}
\end{figure}
\subsection{Overall Performance}
\label{sec:overall}
Figure~\ref{fig:overall} provides a multi-metric comparison across EvoMail and its ablated variants, as well as representative baselines. Four metrics are reported: overall accuracy, F1-score, structured-temporal consistency (STC), and cognitive interpretability (CIM). Several insights emerge. First, the full EvoMail achieves the highest scores across all metrics, with accuracy ($\approx$93\%) and F1 ($\approx$89\%) outperforming all baselines, while also delivering superior STC and CIM. Second, removing the LLM context module (\emph{w/o Context}) causes the largest drop---over 14\% in F1 and severe degradation in STC/CIM---highlighting the necessity of semantic-structural fusion. Third, excluding the query optimizer or memory module also weakens performance: without the query optimizer, EvoMail loses $\approx$3 points in F1 and $\approx$8 points in CIM, while the absence of memory reduces interpretability and robustness, showing the benefit of compressing and reusing failure traces. Finally, classical and shallow graph baselines underperform substantially, with GraphSAGE+MLP lagging behind in STC/CIM and Random Forest lowest overall. These results demonstrate that each architectural component adds complementary benefits, and that EvoMail not only improves accuracy but also achieves more robust and regulator-aligned interpretability compared to neural and classical baselines.
Table~\ref{tab:overall} reports aggregate results across corpora. EvoMail attains the best accuracy (92.8\%), F1 (89.6\%), and interpretability (\cim=0.70). It improves F1 over BERT by +2.7\% and over fusion baselines by +1.1\%, alongside higher \stc.

\begin{table}[t]
\centering
\caption{Overall performance (macro-averaged across datasets). Best in \textbf{bold}.}
\label{tab:overall}
\begin{tabular}{lccccc}
\toprule
Model & Acc & Prec & Rec & F1 & \cim \\
\midrule
NB                 & 88.4 & 84.1 & 82.7 & 83.4 & 0.40 \\
SVM                & 89.3 & 85.6 & 84.7 & 85.1 & 0.41 \\
RF                 & 89.7 & 85.9 & 85.2 & 85.5 & 0.43 \\
TextCNN            & 90.5 & 87.1 & 86.3 & 86.7 & 0.45 \\
BERT               & 92.0 & 87.8 & 86.0 & 86.9 & 0.52 \\
GraphSAGE+MLP      & 90.8 & 87.5 & 87.2 & 87.3 & 0.48 \\
Early Fusion       & 92.2 & 88.2 & 88.6 & 88.4 & 0.55 \\
Mid Fusion         & 92.1 & 88.0 & 88.8 & 88.5 & 0.56 \\
\textbf{EvoMail}   & \textbf{92.8} & \textbf{89.9} & \textbf{89.3} & \textbf{89.6} & \textbf{0.70} \\
\bottomrule
\end{tabular}
\end{table}

\subsection{Effect of Self-Evolution}
\label{sec:selfevolve}
Figure~\ref{fig:improve} tracks F1 performance across ten self-evolution
iterations. EvoMail demonstrates a clear upward trajectory, improving from 
0.78 to nearly 0.89 as it reuses failure traces and integrates adversarially 
generated samples. In contrast, static baselines (BERT and Mid Fusion) 
show only marginal gains and plateau quickly after a few iterations. 
This confirms that EvoMail’s red--blue training loop and compressed 
experience memory enable continual adaptation rather than one-off updates, 
yielding sustained robustness against evolving attack strategies.

\begin{figure}[t]
    \centering
    \includegraphics[width=0.85\linewidth]{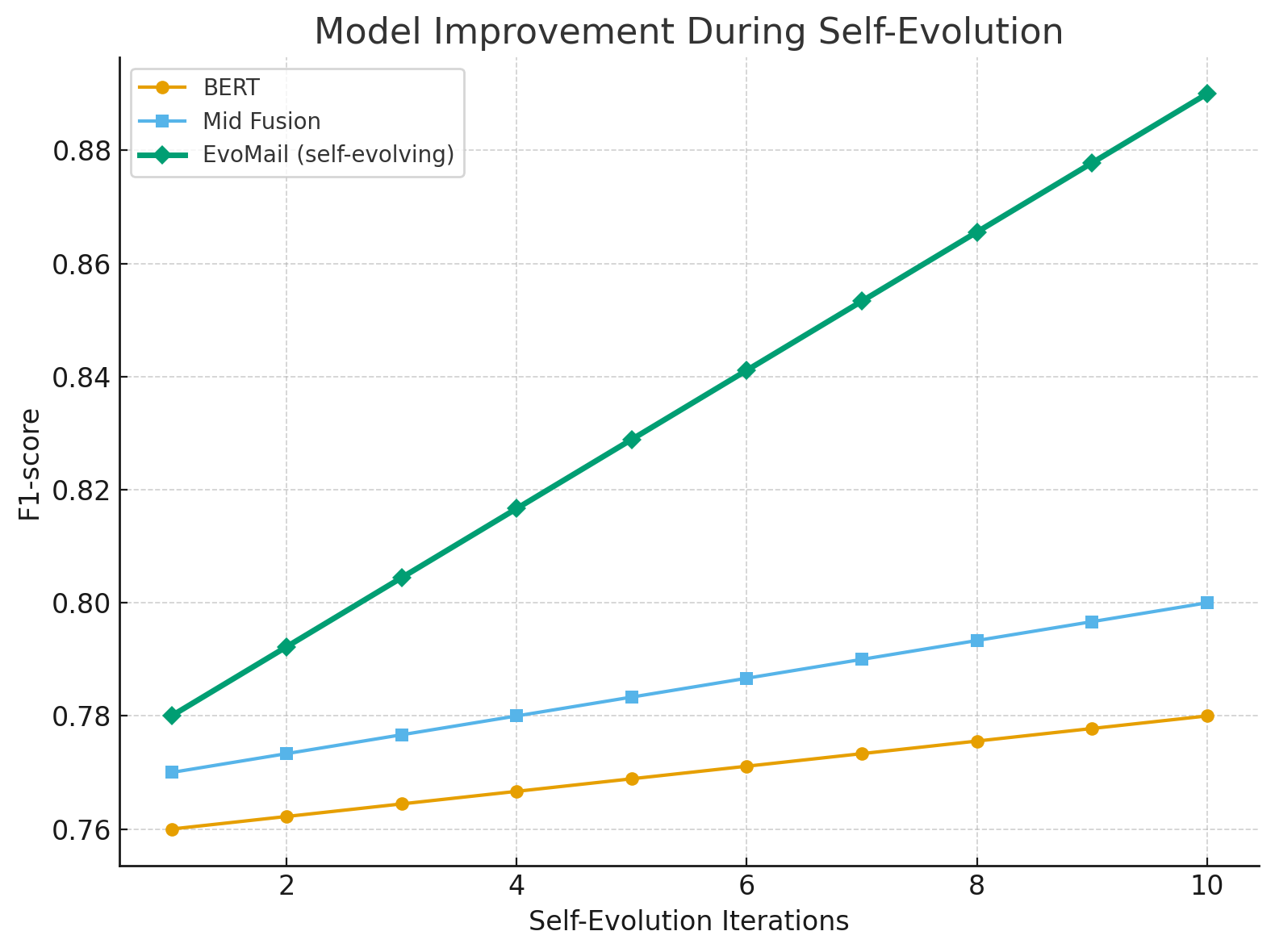}
    \caption{Model improvement during self-evolution iterations. 
    EvoMail shows continuous gains in F1 by reusing failed traces and 
    adversarial examples, whereas static baselines (BERT, Fusion) plateau 
    early with limited improvements.}
    \label{fig:improve}
\end{figure}

\begin{table}[t]
\centering
\caption{Ablation on EvoMail.}
\label{tab:ablation}
\begin{tabular}{lccccc}
\toprule
Variant & Acc & Prec & Rec & F1 & \cim \\
\midrule
Full (EvoMail)     & 92.8 & 89.9 & 89.3 & 89.6 & 0.70 \\
w/o Context        & 83.1 & 76.0 & 74.9 & 75.4 & 0.44 \\
w/o Query Opt.     & 92.0 & 87.0 & 85.7 & 86.3 & 0.62 \\
w/o Memory         & 90.9 & 86.4 & 83.5 & 84.8 & 0.58 \\
\bottomrule
\end{tabular}
\end{table}

\subsection{Ablation Studies}
Table~\ref{tab:ablation} quantifies the contribution of each component in EvoMail. 
The full model achieves the best overall balance across all metrics 
(Acc: 92.8, F1: 89.6, CIM: 0.70). 
Removing the LLM attention module (\emph{w/o Context}) leads to the most 
severe performance degradation, with F1 dropping by 14.2 points and CIM 
falling below 0.45, underscoring the centrality of semantic--structural fusion. 
Excluding the query optimizer (\emph{w/o Query Opt.}) also reduces performance 
(F1: --3.3 points, CIM: --0.08), showing that adaptive neighbor selection 
is important for capturing relevant evidence. 
Similarly, discarding the memory module (\emph{w/o Memory}) lowers both 
F1 (–4.8 points) and CIM (–0.12), highlighting the value of compressing 
and reusing failure traces for stable adaptation. 
Overall, each component provides complementary benefits: 
LLM attention contributes the largest gains in raw accuracy and F1, 
while query optimization and memory enhance robustness and interpretability, 
validating EvoMail’s joint design.

\subsection{Robustness under Distribution Shift}
\label{sec:shift}
Figure~\ref{fig:shift} illustrates the trajectory of AUC across adversarial phases (P1$\rightarrow$P3). 
While all models degrade as spam tactics evolve, EvoMail shows the smallest drop, maintaining an AUC above 0.92 in P3. 
This stability highlights its ability to generalize across distribution shifts and resist tactic drift\citep{Wang2025EnhancingVLM}. 
In particular, EvoMail not only sustains higher accuracy but also preserves detection quality on unseen attacks, reflecting the benefit of its adversarial self-evolution loop.

\begin{figure}[t]
    \centering
    \includegraphics[width=0.85\linewidth]{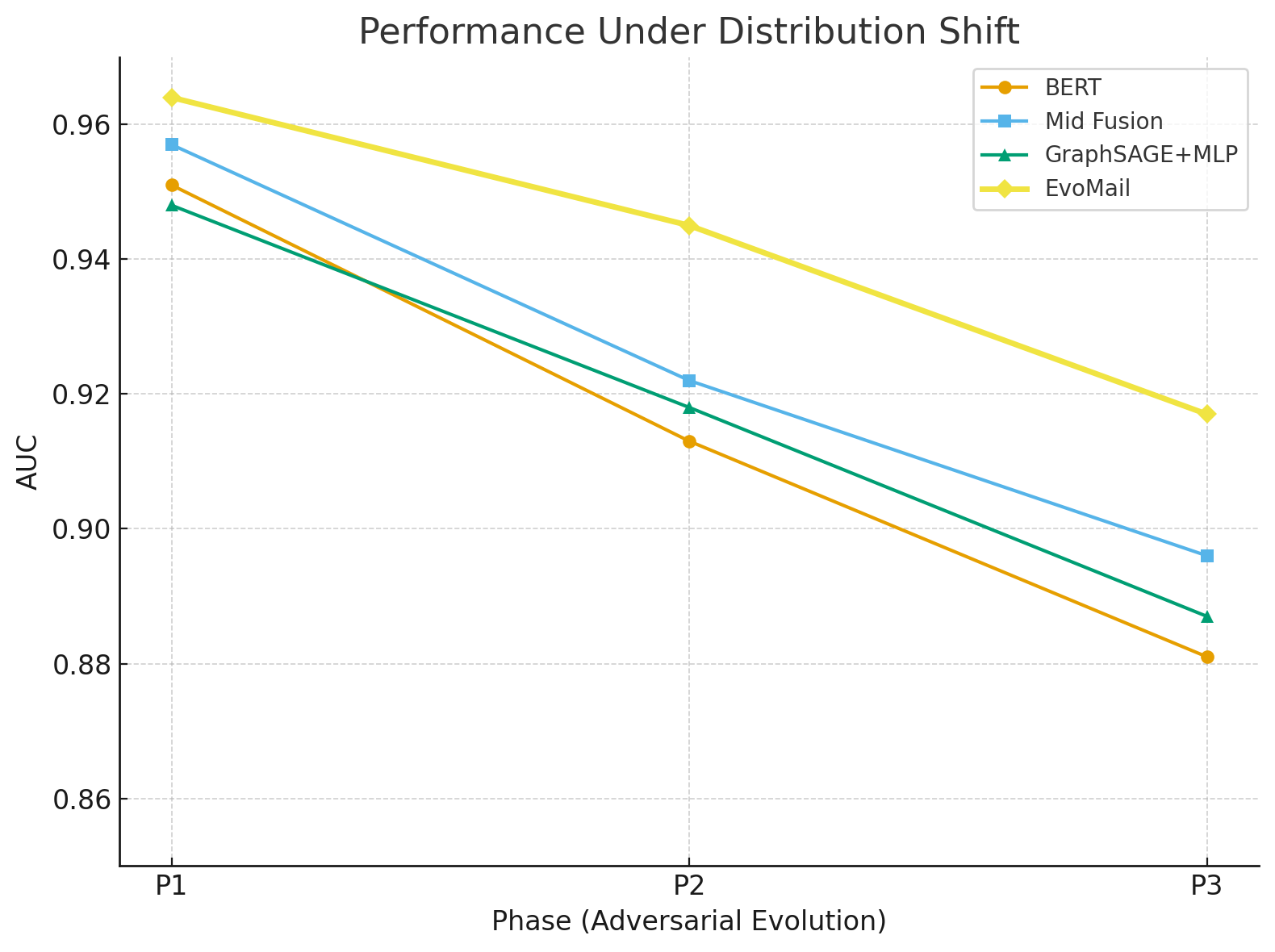}
    \caption{Robustness under distribution shift (Phase 1 $\rightarrow$ Phase 3). 
    EvoMail sustains the highest AUC and novel-attack F1 across phases, showing stronger adaptability than BERT, Fusion, and GraphSAGE.}
    \label{fig:shift}
\end{figure}
\begin{table}[t]
\centering
\caption{Shift robustness (AUC across phases; $\Delta$ is P1--P3). P3 Novel: F1 on 10\% unseen attacks.}
\label{tab:shift}
\resizebox{\linewidth}{!}{%
\begin{tabular}{lccccc}
\toprule
Model & AUC$_{P1}$ & AUC$_{P2}$ & AUC$_{P3}$ & $\Delta$ & P3 Novel \\
\midrule
BERT            & 0.951 & 0.913 & 0.881 & 0.070 & 76.8 \\
Mid Fusion      & 0.957 & 0.922 & 0.896 & 0.061 & 79.4 \\
GraphSAGE+MLP   & 0.948 & 0.918 & 0.887 & 0.061 & 78.6 \\
\textbf{EvoMail}& \textbf{0.964} & \textbf{0.945} & \textbf{0.917} & \textbf{0.047} & \textbf{82.9} \\
\bottomrule
\end{tabular}}
\end{table}
Table~\ref{tab:shift} evaluates performance when attack tactics evolve 
across phases. EvoMail achieves the highest AUC in all phases 
(P1: 0.964, P2: 0.945, P3: 0.917), while the degradation from P1 to P3 
is the smallest among all models ($\Delta=0.047$ vs.\ 0.061--0.070 for others). 
This demonstrates that EvoMail adapts more gracefully to distributional drift. 
Moreover, on novel P3 attacks---emails containing obfuscation patterns not seen 
during training---EvoMail achieves an F1 of 82.9, substantially outperforming 
BERT (76.8), Mid Fusion (79.4), and GraphSAGE+MLP (78.6). 
These results validate the benefits of EvoMail’s self-evolution loop: 
continuous adversarial generation surfaces unseen tactics, while compressed 
memory reuse enables the detector to retain and apply knowledge from past failures. 
Together, these mechanisms reduce the performance erosion typical in static models 
and sustain robustness in adversarially evolving environments.

\subsection{Effect of LLM Head Capacity}
Table~\ref{tab:llm} examines the effect of scaling the LLM attention head capacity. 
As expected, larger configurations yield higher accuracy, F1, and CIM: 
the small model (15M parameters) achieves 91.6\% accuracy and 0.63 CIM, 
while the large model (76M) reaches 92.8\% accuracy and 0.70 CIM. 
The medium model (38M) already delivers strong performance 
(92.3 Acc, 89.1 F1, 0.67 CIM), suggesting that EvoMail gains 
substantial benefits from moderate capacity without requiring very large models. 
This result highlights a favorable trade-off: 
while scale improves interpretability and robustness, 
EvoMail remains competitive even under parameter-efficient settings.
\begin{table}[t]
\centering
\caption{LLM attention head capacity vs.\ performance (illustrative).}
\label{tab:llm}
\begin{tabular}{lcccc}
\toprule
Config & Params (M) & Acc & F1 & \cim \\
\midrule
Small   & 15  & 91.6 & 88.1 & 0.63 \\
Medium  & 38  & 92.3 & 89.1 & 0.67 \\
Large   & 76  & \textbf{92.8} & \textbf{89.6} & \textbf{0.70} \\
\bottomrule
\end{tabular}
\end{table}
\begin{figure}[t]
    \centering
    \includegraphics[width=0.85\linewidth]{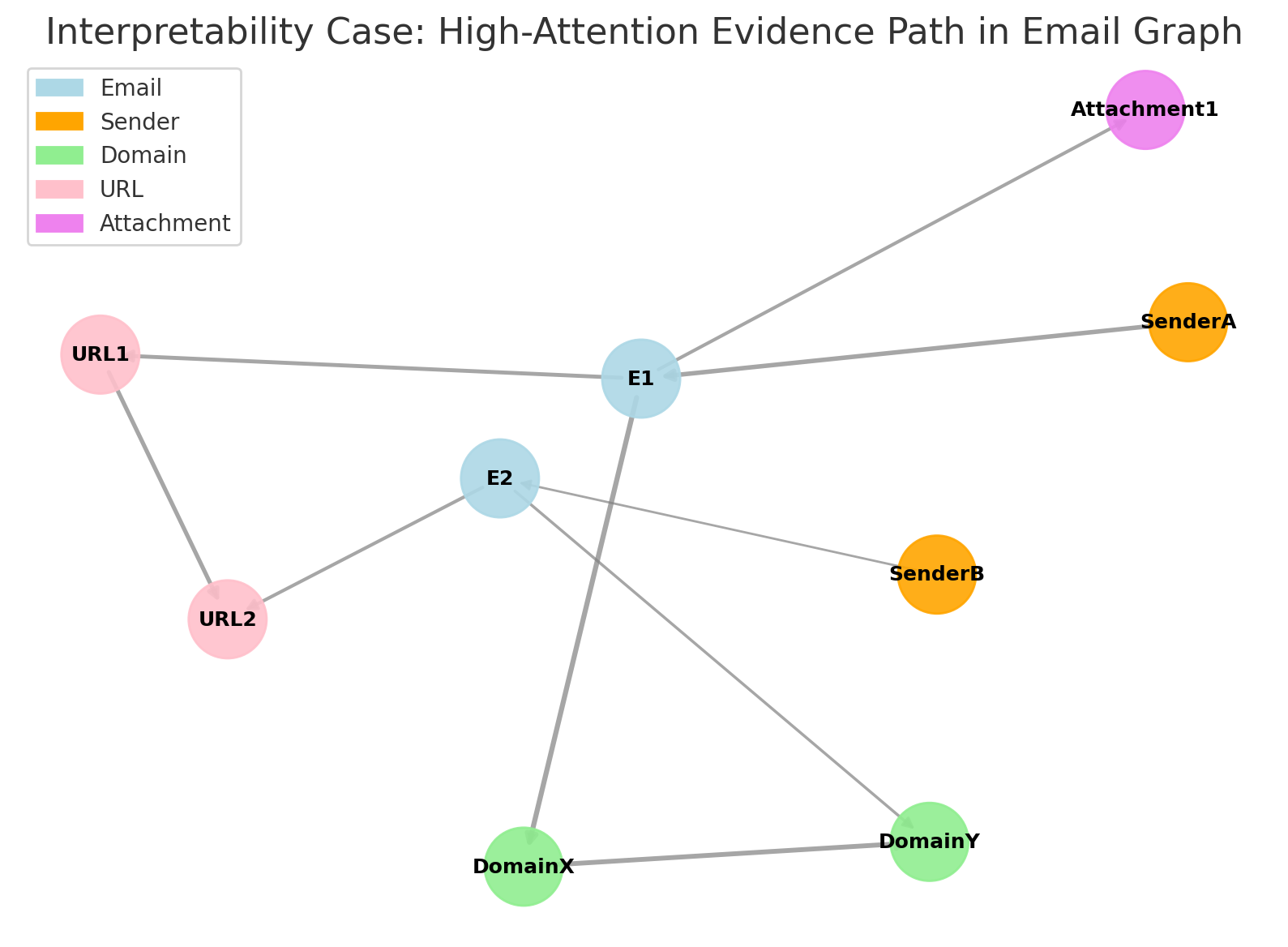}
    \caption{Interpretability case study. 
    The heterogeneous email graph highlights high-attention evidence paths 
    (e.g., suspicious domains, obfuscated URLs, forged headers). 
    EvoMail provides structured audit trails that align with security analyst reasoning.}
    \label{fig:interpret}
\end{figure}
\subsection{Interpretability Case Study}
\label{sec:interpret}

To demonstrate how EvoMail produces transparent reasoning, 
we visualize in Figure~\ref{fig:interpret} a heterogeneous email graph 
with attention-highlighted evidence paths. 
Each node corresponds to an entity type (emails, senders, domains, URLs, attachments), 
and edges denote relations such as \texttt{sent-to}, \texttt{contains}, and \texttt{hosted-on}. 
High-attention paths reveal the cues that drive the final spam/phishing prediction: 
for instance, suspicious domains (DomainX/DomainY), obfuscated URLs (URL1/URL2), 
and a forged attachment–sender linkage. 
By surfacing these traces, EvoMail provides structured audit trails that 
closely align with how human security analysts investigate phishing campaigns, 
bridging detection performance with interpretability\citep{Zhou2024Hallucination}.

\subsection{Case Study: Obfuscated Phishing Variant}
A polymorphic phishing email using zero-width Unicode, forged \texttt{Reply-To}, and a homograph domain evades content-only models. EvoMail elevates attention to SPF failures, URL reputation, and graph anomalies, matching to similar patterns in memory; confidence increases from 0.78 to 0.92 with an interpretable trace (high \cim).


\bibliography{aaai25}

\begin{thebibliography}{65}
\providecommand{\natexlab}[1]{#1}

\bibitem[{Bach et~al.(2015)Bach, Binder, Montavon, Klauschen, Müller, and Samek}]{Bach2015LRP}
Bach, S.; Binder, A.; Montavon, G.; Klauschen, F.; Müller, K.-R.; and Samek, W. 2015.
\newblock On Pixel-Wise Explanations for Non-Linear Classifier Decisions by Layer-Wise Relevance Propagation.
\newblock \emph{PLOS ONE}, 10(7): 1--46.

\bibitem[{Brown et~al.(2020)Brown, Mann, Ryder, Subbiah, Kaplan, Dhariwal, Neelakantan, Shyam, Sastry, Askell, Agarwal, Herbert-Voss, Krueger, Henighan, Child, Ramesh, Ziegler, Wu, Winter, Hesse, Chen, Sigler, Litwin, Gray, Chess, Clark, Berner, McCandlish, Radford, Sutskever, and Amodei}]{Brown2020GPT3}
Brown, T.~B.; Mann, B.; Ryder, N.; Subbiah, M.; Kaplan, J.; Dhariwal, P.; Neelakantan, A.; Shyam, P.; Sastry, G.; Askell, A.; Agarwal, S.; Herbert-Voss, A.; Krueger, G.; Henighan, T.; Child, R.; Ramesh, A.; Ziegler, D.~M.; Wu, J.; Winter, C.; Hesse, C.; Chen, M.; Sigler, E.; Litwin, M.; Gray, S.; Chess, B.; Clark, J.; Berner, C.; McCandlish, S.; Radford, A.; Sutskever, I.; and Amodei, D. 2020.
\newblock Language Models are Few-Shot Learners.
\newblock arXiv:2005.14165.

\bibitem[{Chen et~al.(2025)Chen, Yangfan, Xuxiang, Dong, Jianhui, Tianyu, Arsalan, and Pei}]{Yang2025WCDT}
Chen, Y.; Yangfan, H.; Xuxiang, T.~A.; Dong, C.; Jianhui, W.; Tianyu, S.; Arsalan, H.; and Pei, L. 2025.
\newblock Wcdt: World-Centric Diffusion Transformer for Traffic Scene Generation.
\newblock In \emph{2025 IEEE International Conference on Robotics and Automation (ICRA)}, 6566--6572.

\bibitem[{Cormack(2007)}]{Cormack2007TREC}
Cormack, G.~V. 2007.
\newblock TREC 2007 Spam Track Overview.
\newblock In \emph{Text Retrieval Conference}.

\bibitem[{Devlin et~al.(2019)Devlin, Chang, Lee, and Toutanova}]{Devlin2019BERT}
Devlin, J.; Chang, M.-W.; Lee, K.; and Toutanova, K. 2019.
\newblock BERT: Pre-training of Deep Bidirectional Transformers for Language Understanding.
\newblock arXiv:1810.04805.

\bibitem[{Du et~al.(2024)Du, Yang, Chen, Luo, Jiang, Xin, and Wang}]{Du2024}
Du, Y.; Yang, H.; Chen, M.; Luo, H.; Jiang, J.; Xin, Y.; and Wang, C. 2024.
\newblock Generation, augmentation, and alignment: a pseudo-source domain based method for source-free domain adaptation.
\newblock \emph{Machine Learning}, 113(6): 3611--3631.

\bibitem[{Ede et~al.(2022)Ede, Aghakhani, Spahn, Bortolameotti, Cova, Continella, Steen, Peter, Kruegel, and Vigna}]{Ede2022DEEPCASE}
Ede, T.~v.; Aghakhani, H.; Spahn, N.; Bortolameotti, R.; Cova, M.; Continella, A.; Steen, M.~v.; Peter, A.; Kruegel, C.; and Vigna, G. 2022.
\newblock DEEPCASE: Semi-Supervised Contextual Analysis of Security Events.
\newblock In \emph{2022 IEEE Symposium on Security and Privacy (SP)}, 522--539.

\bibitem[{Fette, Sadeh, and Tomasic(2007)}]{Fette2007Learning}
Fette, I.; Sadeh, N.; and Tomasic, A. 2007.
\newblock Learning to Detect Phishing Emails.
\newblock In \emph{Proceedings of the 16th International Conference on World Wide Web (WWW '07)}, 649--656.

\bibitem[{Ganguli et~al.(2022)Ganguli, Lovitt, Kernion, Askell, Bai, Kadavath, Mann, Perez, Schiefer, Ndousse, Jones, Bowman, Chen, Conerly, DasSarma, Drain, Elhage, El-Showk, Fort, Hatfield-Dodds, Henighan, Hernandez, Hume, Jacobson, Johnston, Kravec, Olsson, Ringer, Tran-Johnson, Amodei, Brown, Joseph, McCandlish, Olah, Kaplan, and Clark}]{Ganguli2022RedTeaming}
Ganguli, D.; Lovitt, L.; Kernion, J.; Askell, A.; Bai, Y.; Kadavath, S.; Mann, B.; Perez, E.; Schiefer, N.; Ndousse, K.; Jones, A.; Bowman, S.; Chen, A.; Conerly, T.; DasSarma, N.; Drain, D.; Elhage, N.; El-Showk, S.; Fort, S.; Hatfield-Dodds, Z.; Henighan, T.; Hernandez, D.; Hume, T.; Jacobson, J.; Johnston, S.; Kravec, S.; Olsson, C.; Ringer, S.; Tran-Johnson, E.; Amodei, D.; Brown, T.; Joseph, N.; McCandlish, S.; Olah, C.; Kaplan, J.; and Clark, J. 2022.
\newblock Red Teaming Language Models to Reduce Harms: Methods, Scaling Behaviors, and Lessons Learned.
\newblock arXiv:2209.07858.

\bibitem[{Ghadimi and Lan(2013)}]{ghadimi2013stochastic}
Ghadimi, S.; and Lan, G. 2013.
\newblock Stochastic first-and zeroth-order methods for nonconvex stochastic programming.
\newblock \emph{SIAM journal on optimization}, 23(4): 2341--2368.

\bibitem[{Goenka, Chawla, and Tiwari(2024)}]{Goenka2024PhishingSurvey}
Goenka, R.; Chawla, M.; and Tiwari, N. 2024.
\newblock A comprehensive survey of phishing: mediums, intended targets, attack and defence techniques and a novel taxonomy.
\newblock \emph{International Journal of Information Security}, 23(2): 819--848.

\bibitem[{Goodfellow, Shlens, and Szegedy(2015)}]{Goodfellow2015Adversarial}
Goodfellow, I.~J.; Shlens, J.; and Szegedy, C. 2015.
\newblock Explaining and Harnessing Adversarial Examples.
\newblock arXiv:1412.6572.

\bibitem[{Hamilton, Ying, and Leskovec(2018)}]{hamilton2018inductiverepresentationlearninglarge}
Hamilton, W.~L.; Ying, R.; and Leskovec, J. 2018.
\newblock Inductive Representation Learning on Large Graphs.
\newblock arXiv:1706.02216.

\bibitem[{He et~al.(2015)He, Zhang, Ren, and Sun}]{He2015ResNet}
He, K.; Zhang, X.; Ren, S.; and Sun, J. 2015.
\newblock Deep Residual Learning for Image Recognition.
\newblock arXiv:1512.03385.

\bibitem[{Khosla, Zhu, and He(2023)}]{Khosla2023MemorySurvey}
Khosla, S.; Zhu, Z.; and He, Y. 2023.
\newblock Survey on Memory-Augmented Neural Networks: Cognitive Insights to AI Applications.
\newblock arXiv:2312.06141.

\bibitem[{Kim(2014)}]{Kim2014CNN}
Kim, Y. 2014.
\newblock Convolutional Neural Networks for Sentence Classification.
\newblock arXiv:1408.5882.

\bibitem[{Kirkpatrick et~al.(2017)Kirkpatrick, Pascanu, Rabinowitz, Veness, Desjardins, Rusu, Milan, Quan, Ramalho, Grabska-Barwinska, Hassabis, Clopath, Kumaran, and Hadsell}]{Kirkpatrick2017EWC}
Kirkpatrick, J.; Pascanu, R.; Rabinowitz, N.; Veness, J.; Desjardins, G.; Rusu, A.~A.; Milan, K.; Quan, J.; Ramalho, T.; Grabska-Barwinska, A.; Hassabis, D.; Clopath, C.; Kumaran, D.; and Hadsell, R. 2017.
\newblock Overcoming catastrophic forgetting in neural networks.
\newblock \emph{Proceedings of the National Academy of Sciences}, 114(13): 3521--3526.

\bibitem[{Li et~al.(2020)Li, Du, Ji, Zhang, Lu, Yang, and Wang}]{Li2020TEXTSHIELD}
Li, J.; Du, T.; Ji, S.; Zhang, R.; Lu, Q.; Yang, M.; and Wang, T. 2020.
\newblock TEXTSHIELD: robust text classification based on multimodal embedding and neural machine translation.
\newblock In \emph{Proceedings of the 29th USENIX Conference on Security Symposium}, SEC'20. USA: USENIX Association.
\newblock ISBN 978-1-939133-17-5.

\bibitem[{Li et~al.(2025)Li, Li, Liu, and Xu}]{Li2025Anomaly}
Li, Y.; Li, J.; Liu, X.; and Xu, L. 2025.
\newblock Deep anomaly detection on attributed networks by graph update.
\newblock \emph{Computing}, 107(6): 134.

\bibitem[{Li et~al.(2024)Li, Xiong, Kong, Bian, Wang, Chen, and Yin}]{Li2024GS2P}
Li, Y.; Xiong, H.; Kong, L.; Bian, J.; Wang, S.; Chen, G.; and Yin, D. 2024.
\newblock GS2P: a generative pre-trained learning to rank model with over-parameterization for web-scale search.
\newblock \emph{Machine Learning}, 113(8): 5331--5349.

\bibitem[{Li et~al.(2023{\natexlab{a}})Li, Xiong, Kong, Wang, Sun, Chen, Chen, and Yin}]{Li2023LtrGCN}
Li, Y.; Xiong, H.; Kong, L.; Wang, S.; Sun, Z.; Chen, H.; Chen, G.; and Yin, D. 2023{\natexlab{a}}.
\newblock LtrGCN: Large-Scale Graph Convolutional Networks-Based Learning to Rank for Web Search.
\newblock In De~Francisci~Morales, G.; Perlich, C.; Ruchansky, N.; Kourtellis, N.; Baralis, E.; and Bonchi, F., eds., \emph{Machine Learning and Knowledge Discovery in Databases: Applied Data Science and Demo Track}, 635--651. Cham: Springer Nature Switzerland.
\newblock ISBN 978-3-031-43427-3.

\bibitem[{Li et~al.(2023{\natexlab{b}})Li, Xiong, Kong, Zhang, Dou, and Chen}]{Li2023MRLtr}
Li, Y.; Xiong, H.; Kong, L.; Zhang, R.; Dou, D.; and Chen, G. 2023{\natexlab{b}}.
\newblock Meta Hierarchical Reinforced Learning to Rank for Recommendation: A Comprehensive Study in MOOCs.
\newblock In Amini, M.-R.; Canu, S.; Fischer, A.; Guns, T.; Kralj~Novak, P.; and Tsoumakas, G., eds., \emph{Machine Learning and Knowledge Discovery in Databases}, 302--317. Cham: Springer Nature Switzerland.
\newblock ISBN 978-3-031-26422-1.

\bibitem[{Li et~al.(2023{\natexlab{c}})Li, Xiong, Wang, Kong, Liu, Li, Bian, Wang, Chen, Dou, and Yin}]{Li2023COLTR}
Li, Y.; Xiong, H.; Wang, Q.; Kong, L.; Liu, H.; Li, H.; Bian, J.; Wang, S.; Chen, G.; Dou, D.; and Yin, D. 2023{\natexlab{c}}.
\newblock COLTR: Semi-Supervised Learning to Rank With Co-Training and Over-Parameterization for Web Search.
\newblock \emph{IEEE Transactions on Knowledge and Data Engineering}, 35(12): 12542--12555.

\bibitem[{Liang et~al.(2025)Liang, Tao, Xia, Wang, Li, Wang, He, Yang, Shi, Wang, Zhang, and Wang}]{LIANG2025130470}
Liang, X.; Tao, M.; Xia, Y.; Wang, J.; Li, K.; Wang, Y.; He, Y.; Yang, J.; Shi, T.; Wang, Y.; Zhang, M.; and Wang, X. 2025.
\newblock {SAGE: Self-evolving Agents with Reflective and Memory-augmented Abilities}.
\newblock \emph{Neurocomputing}, 647: 130470.

\bibitem[{Liang, Ding, and Fu(2021)}]{Liang2021GNNSurvey}
Liang, Z.; Ding, H.; and Fu, W. 2021.
\newblock A Survey on Graph Neural Networks for Recommendation.
\newblock In \emph{2021 International Conference on Culture-oriented Science \& Technology (ICCST)}, 383--386.

\bibitem[{Lin et~al.(2021)Lin, Liu, Divakaran, Ng, Chan, Lu, Si, Zhang, and Dong}]{Lin2021Phishpedia}
Lin, Y.; Liu, R.; Divakaran, D.~M.; Ng, J.~Y.; Chan, Q.~Z.; Lu, Y.; Si, Y.; Zhang, F.; and Dong, J.~S. 2021.
\newblock Phishpedia: A Hybrid Deep Learning Based Approach to Visually Identify Phishing Webpages.
\newblock In \emph{30th USENIX Security Symposium (USENIX Security 21)}, 3793--3810. USENIX Association.
\newblock ISBN 978-1-939133-24-3.

\bibitem[{Liu et~al.(2025)Liu, Zhao, Zhuo, Lin, Xin, Li, Qin, Qiao, Li, and Gao}]{liu2025luminamgpt}
Liu, D.; Zhao, S.; Zhuo, L.; Lin, W.; Xin, Y.; Li, X.; Qin, Q.; Qiao, Y.; Li, H.; and Gao, P. 2025.
\newblock {Lumina-mGPT: Illuminate Flexible Photorealistic Text-to-Image Generation with Multimodal Generative Pretraining}.
\newblock arXiv:2408.02657.

\bibitem[{Liu et~al.(2024)Liu, Liu, Shi, Xu, Huang, Xin, and Yin}]{Liu2024SparseTuning}
Liu, T.; Liu, X.; Shi, L.; Xu, Z.; Huang, S.; Xin, Y.; and Yin, Q. 2024.
\newblock {Sparse-Tuning: Adapting Vision Transformers with Efficient Fine-tuning and Inference}.
\newblock arXiv:2405.14700.

\bibitem[{Lopez-Paz and Ranzato(2022)}]{Lopez-Paz2022GEM}
Lopez-Paz, D.; and Ranzato, M. 2022.
\newblock Gradient Episodic Memory for Continual Learning.
\newblock arXiv:1706.08840.

\bibitem[{Lucic et~al.(2022)Lucic, ter Hoeve, Tolomei, de~Rijke, and Silvestri}]{Lucic2022CF-GNNExplainer}
Lucic, A.; ter Hoeve, M.; Tolomei, G.; de~Rijke, M.; and Silvestri, F. 2022.
\newblock {CF-GNNExplainer}: Counterfactual Explanations for Graph Neural Networks.
\newblock arXiv:2102.03322.

\bibitem[{Lundberg and Lee(2017)}]{Lundberg2017SHAP}
Lundberg, S.; and Lee, S.-I. 2017.
\newblock A Unified Approach to Interpreting Model Predictions.
\newblock arXiv:1705.07874.

\bibitem[{Lyu et~al.(2025)Lyu, Li, Zhu, Xu, Vincent~Poor, and Cui}]{Lyu2025Rethinking}
Lyu, Z.; Li, Y.; Zhu, G.; Xu, J.; Vincent~Poor, H.; and Cui, S. 2025.
\newblock Rethinking Resource Management in Edge Learning: A Joint Pre-Training and Fine-Tuning Design Paradigm.
\newblock \emph{IEEE Transactions on Wireless Communications}, 24(2): 1584--1601.

\bibitem[{Metsis, Androutsopoulos, and Paliouras(2006)}]{Metsis2006Spam}
Metsis, V.; Androutsopoulos, I.; and Paliouras, G. 2006.
\newblock Spam Filtering with Naive Bayes - Which Naive Bayes?
\newblock In \emph{Proceedings of the 3rd Conference on Email and Anti-Spam (CEAS)}.

\bibitem[{Palmer et~al.(2024)Palmer, Parry, Harrold, and Willis}]{Palmer2024CyberSurvey}
Palmer, G.; Parry, C.; Harrold, D. J.~B.; and Willis, C. 2024.
\newblock Deep Reinforcement Learning for Autonomous Cyber Defence: A Survey.
\newblock arXiv:2310.07745.

\bibitem[{Qian et~al.(2023)Qian, Tang, Yang, Liang, and Liu}]{Qian2023LLM-Molecule}
Qian, C.; Tang, H.; Yang, Z.; Liang, H.; and Liu, Y. 2023.
\newblock Can Large Language Models Empower Molecular Property Prediction?
\newblock arXiv:2307.07443.

\bibitem[{Ren et~al.(2024)Ren, Tang, Yin, Chawla, and Huang}]{Ren2024LLM-Graph-Survey}
Ren, X.; Tang, J.; Yin, D.; Chawla, N.; and Huang, C. 2024.
\newblock A Survey of Large Language Models for Graphs.
\newblock In \emph{Proceedings of the 30th ACM SIGKDD Conference on Knowledge Discovery and Data Mining}, KDD ’24, 6616--6626. ACM.

\bibitem[{Ribeiro, Singh, and Guestrin(2016)}]{Ribeiro2016LIME}
Ribeiro, M.~T.; Singh, S.; and Guestrin, C. 2016.
\newblock "Why Should I Trust You?": Explaining the Predictions of Any Classifier.
\newblock arXiv:1602.04938.

\bibitem[{Rossi et~al.(2020)Rossi, Chamberlain, Frasca, Eynard, Monti, and Bronstein}]{Rossi2020TGN}
Rossi, E.; Chamberlain, B.; Frasca, F.; Eynard, D.; Monti, F.; and Bronstein, M. 2020.
\newblock Temporal Graph Networks for Deep Learning on Dynamic Graphs.
\newblock arXiv:2006.10637.

\bibitem[{Schlichtkrull et~al.(2017)Schlichtkrull, Kipf, Bloem, van~den Berg, Titov, and Welling}]{Schlichtkrull2017R-GCN}
Schlichtkrull, M.; Kipf, T.~N.; Bloem, P.; van~den Berg, R.; Titov, I.; and Welling, M. 2017.
\newblock Modeling Relational Data with Graph Convolutional Networks.
\newblock arXiv:1703.06103.

\bibitem[{Shchur et~al.(2019)Shchur, Mumme, Bojchevski, and Günnemann}]{Shchur2019Pitfalls}
Shchur, O.; Mumme, M.; Bojchevski, A.; and Günnemann, S. 2019.
\newblock Pitfalls of Graph Neural Network Evaluation.
\newblock arXiv:1811.05868.

\bibitem[{Silver et~al.(2016)Silver, Huang, Maddison, Guez, Sifre, van~den Driessche, Schrittwieser, Antonoglou, Panneershelvam, Lanctot, Dieleman, Grewe, Nham, Kalchbrenner, Sutskever, Lillicrap, Leach, Kavukcuoglu, Graepel, and Hassabis}]{Silver2016AlphaGo}
Silver, D.; Huang, A.; Maddison, C.~J.; Guez, A.; Sifre, L.; van~den Driessche, G.; Schrittwieser, J.; Antonoglou, I.; Panneershelvam, V.; Lanctot, M.; Dieleman, S.; Grewe, D.; Nham, J.; Kalchbrenner, N.; Sutskever, I.; Lillicrap, T.; Leach, M.; Kavukcuoglu, K.; Graepel, T.; and Hassabis, D. 2016.
\newblock Mastering the game of Go with deep neural networks and tree search.
\newblock \emph{Nature}, 529(7587): 484--489.

\bibitem[{Thapa et~al.(2023)Thapa, Tang, Abuadbba, Gao, Camtepe, Nepal, Almashor, and Zheng}]{Thapa2023FL-Phishing}
Thapa, C.; Tang, J.~W.; Abuadbba, A.; Gao, Y.; Camtepe, S.; Nepal, S.; Almashor, M.; and Zheng, Y. 2023.
\newblock Evaluation of Federated Learning in Phishing Email Detection.
\newblock \emph{Sensors}, 23(9).

\bibitem[{Vaswani et~al.(2017)Vaswani, Shazeer, Parmar, Uszkoreit, Jones, Gomez, Kaiser, and Polosukhin}]{vaswani2017attention}
Vaswani, A.; Shazeer, N.; Parmar, N.; Uszkoreit, J.; Jones, L.; Gomez, A.~N.; Kaiser, {\L}.; and Polosukhin, I. 2017.
\newblock Attention is all you need.
\newblock \emph{Advances in neural information processing systems}, 30.

\bibitem[{Veličković et~al.(2018)Veličković, Cucurull, Casanova, Romero, Liò, and Bengio}]{Velickovic2018GAT}
Veličković, P.; Cucurull, G.; Casanova, A.; Romero, A.; Liò, P.; and Bengio, Y. 2018.
\newblock Graph Attention Networks.
\newblock arXiv:1710.10903.

\bibitem[{Vu and Thai(2020)}]{Vu2020PGM-Explainer}
Vu, M.~N.; and Thai, M.~T. 2020.
\newblock {PGM-Explainer}: Probabilistic Graphical Model Explanations for Graph Neural Networks.
\newblock arXiv:2010.05788.

\bibitem[{Wan et~al.(2025{\natexlab{a}})Wan, Dou, Liu, Zhang, Cui, Zhao, Shen, Xiong, Xin, Jiang, Tao, He, Zhang, and Yan}]{Wan2025SRPO}
Wan, Z.; Dou, Z.; Liu, C.; Zhang, Y.; Cui, D.; Zhao, Q.; Shen, H.; Xiong, J.; Xin, Y.; Jiang, Y.; Tao, C.; He, Y.; Zhang, M.; and Yan, S. 2025{\natexlab{a}}.
\newblock {SRPO}: Enhancing Multimodal LLM Reasoning via Reflection-Aware Reinforcement Learning.
\newblock arXiv:2506.01713.

\bibitem[{Wan et~al.(2025{\natexlab{b}})Wan, Wu, Zhang, Xin, Tao, Zhu, Wang, Luo, Xiong, Wang, and Zhang}]{wan2025d2o}
Wan, Z.; Wu, X.; Zhang, Y.; Xin, Y.; Tao, C.; Zhu, Z.; Wang, X.; Luo, S.; Xiong, J.; Wang, L.; and Zhang, M. 2025{\natexlab{b}}.
\newblock {D2O: Dynamic Discriminative Operations for Efficient Long-Context Inference of Large Language Models}.
\newblock arXiv:2406.13035.

\bibitem[{Wang et~al.(2024{\natexlab{a}})Wang, Zhang, Su, and Zhu}]{Wang2024ContinualSurvey}
Wang, L.; Zhang, X.; Su, H.; and Zhu, J. 2024{\natexlab{a}}.
\newblock A Comprehensive Survey of Continual Learning: Theory, Method and Application.
\newblock \emph{IEEE Transactions on Pattern Analysis and Machine Intelligence}, 46(8): 5362--5383.

\bibitem[{Wang et~al.(2024{\natexlab{b}})Wang, Bian, Li, Li, Mumtaz, Kong, and Xiong}]{Wang2024RNAErnie}
Wang, N.; Bian, J.; Li, Y.; Li, X.; Mumtaz, S.; Kong, L.; and Xiong, H. 2024{\natexlab{b}}.
\newblock Multi-purpose RNA language modelling with motif-aware pretraining and type-guided fine-tuning.
\newblock \emph{Nature Machine Intelligence}, 6(5): 548--557.

\bibitem[{Wang et~al.(2025{\natexlab{a}})Wang, Chen, Wang, Zhou, Zhou, Yao, Zhou, Goldstein, Bhatia, Huang, and Xiao}]{Wang2025EnhancingVLM}
Wang, X.; Chen, J.; Wang, Z.; Zhou, Y.; Zhou, Y.; Yao, H.; Zhou, T.; Goldstein, T.; Bhatia, P.; Huang, F.; and Xiao, C. 2025{\natexlab{a}}.
\newblock Enhancing Visual-Language Modality Alignment in Large Vision Language Models via Self-Improvement.
\newblock arXiv:2405.15973.

\bibitem[{Wang et~al.(2021)Wang, Ji, Shi, Wang, Cui, Yu, and Ye}]{Wang2021HGAN}
Wang, X.; Ji, H.; Shi, C.; Wang, B.; Cui, P.; Yu, P.; and Ye, Y. 2021.
\newblock Heterogeneous Graph Attention Network.
\newblock arXiv:1903.07293.

\bibitem[{Wang et~al.(2025{\natexlab{b}})Wang, He, Wang, Li, Sun, Yin, Zhang, and Wang}]{Xin2025Enhancing}
Wang, Y.; He, Y.; Wang, J.; Li, K.; Sun, L.; Yin, J.; Zhang, M.; and Wang, X. 2025{\natexlab{b}}.
\newblock Enhancing intent understanding for ambiguous prompt: A human–machine co-adaption strategy.
\newblock \emph{Neurocomputing}, 646: 130415.

\bibitem[{Wei, Haghtalab, and Steinhardt(2023)}]{Wei2023Jailbroken}
Wei, A.; Haghtalab, N.; and Steinhardt, J. 2023.
\newblock Jailbroken: How Does LLM Safety Training Fail?
\newblock arXiv:2307.02483.

\bibitem[{Wu et~al.(2021)Wu, Pan, Chen, Long, Zhang, and Yu}]{Wu2021Survey}
Wu, Z.; Pan, S.; Chen, F.; Long, G.; Zhang, C.; and Yu, P.~S. 2021.
\newblock A Comprehensive Survey on Graph Neural Networks.
\newblock \emph{IEEE Transactions on Neural Networks and Learning Systems}, 32(1): 4–24.

\bibitem[{Xin et~al.(2024{\natexlab{a}})Xin, Du, Wang, Lin, and Yan}]{Xin2024VMT-Adapter}
Xin, Y.; Du, J.; Wang, Q.; Lin, Z.; and Yan, K. 2024{\natexlab{a}}.
\newblock VMT-Adapter: Parameter-Efficient Transfer Learning for Multi-Task Dense Scene Understanding.
\newblock \emph{Proceedings of the AAAI Conference on Artificial Intelligence}, 38(14): 16085--16093.

\bibitem[{Xin et~al.(2024{\natexlab{b}})Xin, Du, Wang, Yan, and Ding}]{Xin2024MmAP}
Xin, Y.; Du, J.; Wang, Q.; Yan, K.; and Ding, S. 2024{\natexlab{b}}.
\newblock MmAP: Multi-Modal Alignment Prompt for Cross-Domain Multi-Task Learning.
\newblock \emph{Proceedings of the AAAI Conference on Artificial Intelligence}, 38(14): 16076--16084.

\bibitem[{Xin et~al.(2024{\natexlab{c}})Xin, Luo, Liu, Du, Zhou, Cheng, Lee, Du, Wang, Chen, Liu, Hu, Wan, Zhang, Li, Yi, and Liu}]{Xin2024V-PETL}
Xin, Y.; Luo, S.; Liu, X.; Du, Y.; Zhou, H.; Cheng, X.; Lee, C.; Du, J.; Wang, H.; Chen, M.; Liu, T.; Hu, G.; Wan, Z.; Zhang, R.; Li, A.; Yi, M.; and Liu, X. 2024{\natexlab{c}}.
\newblock V-PETL Bench: A Unified Visual Parameter-Efficient Transfer Learning Benchmark.
\newblock In Globerson, A.; Mackey, L.; Belgrave, D.; Fan, A.; Paquet, U.; Tomczak, J.; and Zhang, C., eds., \emph{Advances in Neural Information Processing Systems}, volume~37, 80522--80535. Curran Associates, Inc.

\bibitem[{Xiong et~al.(2024)Xiong, Bian, Li, Li, Du, Wang, Yin, and Helal}]{Xiong2024SearchLLM}
Xiong, H.; Bian, J.; Li, Y.; Li, X.; Du, M.; Wang, S.; Yin, D.; and Helal, S. 2024.
\newblock When Search Engine Services meet Large Language Models: Visions and Challenges.
\newblock arXiv:2407.00128.

\bibitem[{Ying et~al.(2019)Ying, Bourgeois, You, Zitnik, and Leskovec}]{Ying2019GNNExplainer}
Ying, R.; Bourgeois, D.; You, J.; Zitnik, M.; and Leskovec, J. 2019.
\newblock GNNExplainer: Generating Explanations for Graph Neural Networks.
\newblock arXiv:1903.03894.

\bibitem[{Zhao et~al.(2023)Zhao, Qu, Li, Yan, Liu, Li, Xie, and Tang}]{Zhao2023GLEM}
Zhao, J.; Qu, M.; Li, C.; Yan, H.; Liu, Q.; Li, R.; Xie, X.; and Tang, J. 2023.
\newblock Learning on Large-scale Text-attributed Graphs via Variational Inference.
\newblock In \emph{The Eleventh International Conference on Learning Representations}.

\bibitem[{Zhao et~al.(2025)Zhao, Zhou, Li, Tang, Wang, Hou, Min, Zhang, Zhang, Dong, Du, Yang, Chen, Chen, Jiang, Ren, Li, Tang, Liu, Liu, Nie, and Wen}]{Zhao2025SurveyLLM}
Zhao, W.~X.; Zhou, K.; Li, J.; Tang, T.; Wang, X.; Hou, Y.; Min, Y.; Zhang, B.; Zhang, J.; Dong, Z.; Du, Y.; Yang, C.; Chen, Y.; Chen, Z.; Jiang, J.; Ren, R.; Li, Y.; Tang, X.; Liu, Z.; Liu, P.; Nie, J.-Y.; and Wen, J.-R. 2025.
\newblock A Survey of Large Language Models.
\newblock arXiv:2303.18223.

\bibitem[{Zhou et~al.(2020)Zhou, Cui, Hu, Zhang, Yang, Liu, Wang, Li, and Sun}]{zhou2020graph}
Zhou, J.; Cui, G.; Hu, S.; Zhang, Z.; Yang, C.; Liu, Z.; Wang, L.; Li, C.; and Sun, M. 2020.
\newblock Graph neural networks: A review of methods and applications.
\newblock \emph{AI open}, 1: 57--81.

\bibitem[{Zhou et~al.(2024{\natexlab{a}})Zhou, Cui, Yoon, Zhang, Deng, Finn, Bansal, and Yao}]{Zhou2024Hallucination}
Zhou, Y.; Cui, C.; Yoon, J.; Zhang, L.; Deng, Z.; Finn, C.; Bansal, M.; and Yao, H. 2024{\natexlab{a}}.
\newblock Analyzing and Mitigating Object Hallucination in Large Vision-Language Models.
\newblock arXiv:2310.00754.

\bibitem[{Zhou et~al.(2024{\natexlab{b}})Zhou, Fan, Cheng, Yang, Chen, Cui, Wang, Li, Zhang, and Yao}]{Zhou2024CalibratedVLM}
Zhou, Y.; Fan, Z.; Cheng, D.; Yang, S.; Chen, Z.; Cui, C.; Wang, X.; Li, Y.; Zhang, L.; and Yao, H. 2024{\natexlab{b}}.
\newblock Calibrated Self-Rewarding Vision Language Models.
\newblock In Globerson, A.; Mackey, L.; Belgrave, D.; Fan, A.; Paquet, U.; Tomczak, J.; and Zhang, C., eds., \emph{Advances in Neural Information Processing Systems}, volume~37, 51503--51531. Curran Associates, Inc.

\bibitem[{Zhuo et~al.(2025)Zhuo, Zhao, Paul, Liao, Zhang, Xin, Gao, Elhoseiny, and Li}]{Zhuo2025Reflection}
Zhuo, L.; Zhao, L.; Paul, S.; Liao, Y.; Zhang, R.; Xin, Y.; Gao, P.; Elhoseiny, M.; and Li, H. 2025.
\newblock {From Reflection to Perfection: Scaling Inference-Time Optimization for Text-to-Image Diffusion Models via Reflection Tuning}.
\newblock arXiv:2504.16080.

\end{thebibliography}


\end{document}